\documentclass[11pt, reqno]{amsart}
\makeatletter
    \def\@@and{et}
\makeatother

\usepackage[T1]{fontenc}
\usepackage{graphicx}
\usepackage{xcolor}
\usepackage{booktabs}
\usepackage[misc]{ifsym}
\usepackage{amssymb}
 \usepackage[all]{xy}
\usepackage[misc]{ifsym}
\usepackage{ifdraft}
\usepackage{amsmath,bm,xfrac,mathabx,mathtools,stackengine}
\usepackage{algorithm}
\usepackage{algpseudocode}
\usepackage[colorlinks=true, pdfstartview=FitV, linkcolor=blue, citecolor=blue, urlcolor=blue]{hyperref}
\usepackage[justification=centering]{caption}
\usepackage{comment}
\usepackage{mathrsfs}

\usepackage{xr}
\usepackage{mwe}

\newcommand{\delete}[1]{}
\newcommand{\on}[1]{\operatorname{#1}}

\title[M\'Ethode de quadrature pour les PINNs fond\'Ee sur la hessienne des r\'Esiduels]{M\'Ethode de quadrature pour les PINNs fond\'Ee th\'Eoriquement sur la hessienne des r\'Esiduels}
\author[A. CARADOT]{Antoine CARADOT$^1$}
\author[R. EMONET]{R\'emi Emonet$^1$}
\author[A. HABRARD]{Amaury Habrard$^1$}
\author[A.R. Mezidi]{Abdel-Rahim Mezidi$^1$}
\author[M. SEBBAN]{Marc Sebban$^1$}

\address{$^1$Laboratoire Hubert Curien, UniversitŽ Jean Monnet, Saint-\'Etienne, France}
\date{24 juin 2025}

\keywords{PINN ; Points de collocation ; ƒchantillonnage adaptatif ; MŽthode de quadrature.}
\thanks{2020 {\it Mathematics Subject Classification:} Primary 68T07, 41A55; Secondary 62D05, 35A25}

\theoremstyle{plain}
\newtheorem{theorem}{ThŽorme}[section]
\newtheorem{lemma}{Lemme}[section]
\newtheorem{proposition}{Proposition}[section]

\theoremstyle{remark}

\allowdisplaybreaks

\defŽ{\'e}
\defƒ{\'E}
\defË{\`A}
\def{\c{c}}
\def™{\^{o}}
\defž{\^{u}}
\def{\`{u}}
\def‰{\^{a}}
\def"{\^{i}}
\defˆ{\`{a}}
\def{\`{e}}
\def{\^{e}}
\def•{\"{i}}
\defš{\"{o}}
\defÏ{\oe}

\usepackage[utf8]{inputenc}
\usepackage[french,english]{babel}
\makeatletter
\newenvironment{abstracts}{%
  \ifx\maketitle\relax
    \ClassWarning{\@classname}{Abstract should precede
      \protect\maketitle\space in AMS document classes; reported}%
  \fi
  \global\setbox\abstractbox=\vtop \bgroup
    \normalfont\Small
    \list{}{\labelwidth\z@
      \leftmargin3pc \rightmargin\leftmargin
      \listparindent\normalparindent \itemindent\z@
      \parsep\z@ \@plus\p@
      
      \itemsep\medskipamount
    }%
}{%
  \endlist\egroup
  \ifx\@setabstract\relax \@setabstracta \fi
}

\newcommand{\abstractin}[1]{%
  \otherlanguage{#1}%
  \item[\hskip\labelsep\scshape\abstractname.]%
}
\makeatother

\begin{document}



\begin{abstracts}
\abstractin{french}
Les rŽseaux de neurones informŽs par la physique (PINNs) sont apparus comme un moyen efficace d'apprendre des solveurs d'EDP en incorporant le modle physique dans la fonction de perte et en minimisant ses rŽsiduels par diffŽrenciation automatique ˆ des points dits de collocation. Originellement sŽlectionnŽs de manire uniforme, le choix de ces derniers a fait l'objet d'avancŽes rŽcentes en Žchantillonnage adaptatif. Nous proposons ici une nouvelle mŽthode de quadrature pour l'approximation d'intŽgrale basŽe sur la hessienne de la fonction considŽrŽe, pour laquelle nous dŽrivons une borne d'erreur d'approximation. Nous exploitons cette information de second ordre pour guider la sŽlection des points de collocation durant l'apprentissage des PINNs.


\abstractin{english}
Physics-informed Neural Networks (PINNs) have emerged as an efficient way to learn surrogate neural solvers of PDEs by embedding the physical model in the loss function and minimizing its residuals using automatic differentiation at so-called collocation points. Originally uniformly sampled, the choice of the latter has been the subject of recent advances leading to adaptive sampling refinements. In this paper, we propose a new quadrature method for approximating definite integrals based on the hessian of the considered function, and that we leverage to guide  the selection of the collocation points during the training process of PINNs.
\end{abstracts}
\maketitle


\section{Introduction}

\delete{
Incorporating domain knowledge into machine learning algorithms has become a widespread strategy for managing ill-posed problems, data scarcity and solution consistency. Indeed, ignoring the fundamental principles of the underlying theory may lead to, yet optimal, implausible solutions yielding poor generalization and predictions with a high level of uncertainty. Embedding domain knowledge has been shown to be useful  when used at different levels of the learning process for (i) constraining/regularizing the optimization problem, (ii) designing suitable theory-guided loss functions, (iii) initializing  models with meaningful parameters, (iv) designing consistent neural network  architectures, or (v) building (theory/data)-driven hybrid models. In this context, physics is probably  the scientific domain that has benefited the most during the past years from  advances in the so-called {\it Physics-informed Machine Learning} (PiML) field \cite{karniadakis_physics-informed_2021} by leveraging physical laws, typically in the form of Partial Differential Equations (PDEs) that govern some underlying dynamical system. This new line of research led to a novel generation of deep-learning architectures, including Neural ODE \cite{chen2019neural}, PINN \cite{raissi2019physics}, FNO \cite{LiKALBSA21}, PINO \cite{li2023physicsinformed},  PDE-Net \cite{long2018pdenet}, etc.
}

MalgrŽ des avancŽes scientifiques importantes en simulation numŽrique, rŽsoudre efficacement des EDPs demeure un problme complexe et cožteux. En incorporant l'Žquation dans la fonction de perte et en minimisant ses rŽsiduels en des points de {\it collocation}, les rŽseaux de neurones informŽs par la physique (PINNs) \cite{raissi2019physics} sont apparus comme une solution sŽduisante pour apprendre efficacement des solveurs neuronaux. MalgrŽ leur efficacitŽ, les PINNs sont encore mal compris, et il est crucial d'Žtudier leurs fondements thŽoriques ainsi que leurs propriŽtŽs algorithmiques afin d'avoir une bonne comprŽhension de leurs capacitŽs et  limites. En effet, des Žtudes rŽcentes ont montrŽ que les PINNs peuvent tre sujets ˆ des comportements pathologiques, conduisant ˆ des rŽsiduels nuls, donc plausibles du point de vue physique, mais correspondant ˆ des solutions incorrectes \cite{Chandrajit2021,doumeche2023convergence}. 
La caractŽrisation de ces « modes d'Žchec » 
a menŽ ˆ une recherche active abordant la question sous deux angles diffŽrents : un premier thŽorique visant ˆ Žtablir dans un contexte d'Žchantillonnage uniforme des points de collocation (e.g., grille uniforme Žquidistante ou tirage alŽatoire uniforme), des garanties de consistance et de convergence sous forme de bornes d'estimation, ou d'approximation (e.g.,  \cite{doumeche2023convergence,girault:hal-04518335,deryck2023error}) ; un deuxime amŽliorant l'Žchantillonnage des points de collocation. Plut™t que de les tirer de manire uniforme, plusieurs stratŽgies  ont ŽmergŽ dans la littŽrature, suggŽrant d'orienter la sŽlection des points au cours de l'apprentissage en fonction de l'amplitude ou du gradient des rŽsiduels de l'EDP. Ceci a donnŽ lieu ˆ une nouvelle famille de mŽthodes d'Žchantillonnage adaptatif pour les PINNs (e.g., \cite{pmlr-v202-daw23a,subramanian2023,Chenxi2022,gPINN2021}). 
Cependant, il convient de noter que, bien que ces mŽthodes aient dŽmontrŽ des performances remarquables en pratique, elles ont en commun l'absence de garanties thŽoriques quant ˆ leur avantage par rapport ˆ un Žchantillonnage uniforme. \\
L'objectif de ce papier est de rŽpondre ˆ cette limitation. Partant du fait que la minimisation d'une perte empirique en apprentissage  peut tre abordŽe d'un point de vue mathŽmatique ˆ travers l'approximation de l'intŽgrale d'une fonction $f$, nous proposons une nouvelle rgle de quadrature simple basŽe sur la hessienne de $f$. Nous dŽrivons une borne supŽrieure sur l'erreur d'approximation  et montrons sa meilleure prŽcision par rapport ˆ celle issue d'une grille uniforme rŽgulire. Ce rŽsultat thŽorique nous conduit ˆ concevoir une  mŽthode d'Žchantillonnage adaptatif pour les PINNs, o $f$ prend la forme de la fonction de perte basŽe sur les rŽsiduels. Cette stratŽgie sŽlectionne les points de collocation dans le domaine spatio-temporel aux endroits o la hessienne varie le plus. Les expŽriences menŽes sur deux EDPs en 2D mettent en Žvidence les propriŽtŽs intŽressantes de notre mŽthode.

\delete{
The objective of this paper is to bridge the gap by providing two new methodological contributions: (i) Recalling that minimizing an empirical loss in machine learning can be approached from a mathematical perspective as the approximation of the integral of some function $f$, we propose a new quadrature rule based on a simple trapezoidal interpolation and information about the second-order derivative $f{''}$. We derive an upper bound on the approximation error and show its tightness compared to that of issued from an equispaced uniform grid. This theoretical result is supported by several experiments. (ii) This finding prompts us to design a new theoretically founded adaptive sampling method for PINNs where $f$ takes the form of the residual-based loss function. This  strategy 
selects collocation points in the spatio-temporal domain where $f{''}$ varies the most. Experiments conducted on 1D and 2D PDEs highlight the interesting properties of our method.

The rest of this paper is organized as follows: in Section~\ref{sec:BG}, we introduce the  necessary background and related work; Section~\ref{sec:quadrature} is devoted to the presentation of our refined quadrature method and the upper bound derived on the total approximation error. In Section~\ref{sec:PINN}, we leverage our quadrature method to propose a new adaptive sampling method for PINNs and test it on 1D and 2D PDEs.
}

\section{Notations} \label{sec:BG}
On considre les EDPs de la forme $\frac{\partial u}{\partial t}+{\mathcal N[}u]=0,$ 
o ${\mathcal N[}\cdot]$ est un opŽrateur diffŽrentiel en temps et en espace, et o  $u(t,\mathbf{x})$ est la solution avec $t \in [0,T]$ et  $\mathbf{x} \in \Omega$. Cette Žquation est gŽnŽralement accompagnŽe de conditions initiales et limites : $\forall \mathbf{x} \in \Omega,$ ${\mathcal I}[u](0,\mathbf{x})=0$, et $\forall \mathbf{x} \in \partial \Omega, t \in [0,T]$, ${\mathcal B}[u](t,\mathbf{x})=0$, o ${\mathcal B}$ est l'opŽrateur de frontire appliquŽ ˆ $\partial \Omega$ et  ${\mathcal I}$ est l'opŽrateur initial ˆ $t=0$. 
L'objectif d'un PINN \cite{raissi2019physics} est d'apprendre une approximation $u_{\theta}(t,\mathbf{x})$ de la solution $u(t,\mathbf{x})$ en optimisant les paramtres $\theta$ d'un rŽseau de neurones par la minimisation d'une fonction de perte ${\mathcal L}(\theta)$ (voir Problme~\ref{eq:loss}) composŽe des termes non nŽgatifs suivants:
\delete{
\begin{eqnarray}
{\mathcal L}_{\mathcal N}(\theta) & = & \int_{[0,T] \times \Omega} \left(\frac{\partial u_{\theta}}{\partial t}+{\mathcal N[}u_{\theta};\phi]\right)^2 dtd\mathbf{x}  \label{eq:col} \nonumber
\end{eqnarray}
\begin{align}
\scalebox{0.9}{${\mathcal L}_{\mathcal I}(\theta)  =  \int_{\Omega} ({\mathcal I}[u_{\theta}](0,\mathbf{x}))^2 d\mathbf{x}, \ {\mathcal L}_{\mathcal B}(\theta)  =  \int_{\partial \Omega} \left( {\mathcal B}[u_{\theta}](t,\mathbf{x}) \right)^2 dtd\mathbf{x}$} \nonumber 
\end{align}
}
\begin{eqnarray}
{\mathcal L}_{\mathcal N}(\theta) & = & \int_{[0,T] \times \Omega} \left(\frac{\partial u_{\theta}}{\partial t}+{\mathcal N[}u_{\theta}]\right)^2 dtd\mathbf{x}  \label{eq:col} \nonumber \\
{\mathcal L}_{\mathcal I}(\theta) & = & \int_{\Omega} ({\mathcal I}[u_{\theta}](0,\mathbf{x}))^2 d\mathbf{x}  \label{eq:init} \nonumber \\
{\mathcal L}_{\mathcal B}(\theta) & = & \int_{[0,T] \times \partial \Omega} \left( {\mathcal B}[u_{\theta}](t,\mathbf{x}) \right)^2 dtd\mathbf{x} \label{eq:bound} \nonumber 
\end{eqnarray}
\delete{
\begin{align}
\scalebox{0.9}{$\underset{\theta}{\min} \hspace{0.1cm} {\mathcal L}(\theta) = \underset{\theta}{\min} (  {\mathcal L}_{\mathcal N}(\theta)+\lambda_1{\mathcal L}_{\mathcal I}(\theta)+\lambda_2{\mathcal L}_{\mathcal B}(\theta)+\lambda_3R(\theta))$,} \label{eq:loss}
\end{align}
}
\begin{align}
 \underset{\theta}{\min} \hspace{0.1cm} {\mathcal L}(\theta)   = \underset{\theta}{\min} (  {\mathcal L}_{\mathcal N}(\theta)+\lambda_1{\mathcal L}_{\mathcal I}(\theta)+\lambda_2{\mathcal L}_{\mathcal B}(\theta)+\lambda_3R(\theta)), \label{eq:loss}
\end{align}
o $\lambda_1,\lambda_2,\lambda_3$ sont des hyperparamtres et $R(\theta)$ est un terme de rŽgularisation. En pratique, les intŽgrales ${\mathcal L}_{\mathcal N}(\theta)$, ${\mathcal L}_{\mathcal I}(\theta)$, et ${\mathcal L}_{\mathcal B}(\theta)$ sont approximŽes par des espŽrances calculŽes ˆ partir de $N_{\mathcal N}$ points de collocation, $N_{\mathcal I}$ points initiaux et $N_{\mathcal B}$ points aux bords, respectivement.\\
D'un point de vue mathŽmatique, les intŽgrales de Eq.~\eqref{eq:loss} s'obtiennent en approximant l'intŽgrale d'une fonction $f:\mathscr{D}\to \mathbb{R}$ ˆ partir de $N$ mesures de l'intŽgrande par une quadrature numŽrique comme suit : $\sum_{i=1}^N w_if(\mathbf{x}_i) \approx \int_{\mathscr{D}} f(\mathbf{x})d\mathbf{x},$ o $ w_i\geq 0$ sont des poids de quadrature.\\

Dans la section suivante, nous prŽsentons une nouvelle mŽthode de quadrature reposant sur la hessienne de $f$ et dŽrivons une borne supŽrieure sur l'erreur d'approximation, plus serrŽe que celle d'une mŽthode sŽlectionnant de manire rŽgulire les $N=N_{\mathcal N}+N_{\mathcal I}+N_{\mathcal B}$ points de quadrature.

\section{Bornes d'erreurs de quadrature} \label{sec:quadrature}
Soient $a, b \in \mathbb{R}$ et une fonction $f:[a,b] \longrightarrow \mathbb{R}$. Pour rŽsoudre la quadrature de $f$, on choisit $x_0, \dots, x_N \in [a,b]$ et les poids $w_0, \dots, w_N$ peuvent tre obtenus en approximant $f$ par des fonctions polynomiales (mŽthode dite de Newton-Cotes). \delete{There exists another way to proceed where the collocation points are obtained as roots of a particular orthonormal polynomial, but this approach would become to computationally heavy for our purpose.} Afin d'Žviter le phŽnomne de Runge, o l'interpolation polynomiale prŽsente des piques aux bords de l'intervalle ˆ cause de la croissance de $\on{max}_{x \in [a,b]}|f^{(n)}(x)|$ comme une fonction de $n$, on propose ici de contr™ler l'expressivitŽ de l'approximation en utilisant une interpolation par de simples trapzes.

\delete{
In the following, we first present the quadrature rule when the quadrature points are evenly spaced in the domain and recall a known result on the upper bound on the approximation error in this trapezoid-based setting. Then, we introduced a refined quadrature rule which selects the quadrature points where the second-order derivative of $f$ varies the most. The main result of this section takes the form of a tighter upper bound on the total approximation error.
}

\subsection{Quadrature ˆ partir d'une grille uniforme} \label{sec:uniform}
DŽterminons tout d'abord l'erreur qui serait obtenue par une rgle basŽe sur des trapzes construits ˆ partir d'une grille uniforme de points de quadrature. 
On approxime tout d'abord $f$ sur $[x_1, x_2]$, avec $x_1, x_2 \in [a,b]$, par la droite passant par $(x_1, f(x_1))$ et $(x_2, f(x_2))$. On pose $h=x_2-x_1$. L'interpolation $p(x)$ est donnŽe par: 
\[
p(x)=\frac{x-x_1}{h}f(x_2)-\frac{x-x_2}{h}f(x_1).
\]
Comme il s'agit d'un polyn™me de degrŽ 1, la formule d'interpolation de Lagrange implique que $\exists \xi \in [x_1, x_2]$ tel que
\[
f(x)-p(x)=\frac{f''(\xi)}{2}(x-x_1)(x-x_2).
\]

\delete{
Note that while the existence of $\xi$ is guaranteed, we don't know its position within $[z_1, z_2]$. Moreover, by setting $s=\frac{x-z_1}{h}$, we see that $(x-z_1)(x-z_2)=s(s-1)h^2$. It follows that the error $E_{z_1,z_2}=\int_{z_1}^{z_1}f(x)dx-\int_{z_1}^{z_2}p(x)dx$ on $[z_1, z_2]$ satisfies
\[
E_{z_1,z_2}=-\frac{1}{12}h^3f''(\xi).
\]

}
Divisons dŽsormais $[a,b]$ en $N$ sous-intervalles de longueur $h=\frac{b-a}{N}$. On pose $x_0=a$ et $x_i=x_0+hi$ pour $1 \leq i \leq N$. La fonction $f$ est approximŽe sur chaque $[x_i, x_{i+1}]$ par une droite, et donc $p(x)$ devient une fonction linŽaire par morceaux. L'intŽgrale de $f$ peut alors tre approximŽe par l'intŽgrale des $N$ trapzes formŽs par $p(x)$.  En exploitant \cite[Th.20.5.1]{Hamming}, on peut prouver le rŽsultat suivant:

\delete{For each $1 \leq i \leq N$, there exists $\xi_i\in [x_{i-1}, x_{i}]$ such that the total integration error $E_{\mathrm{tot},\mathrm{unif}}$ on $[x_0, x_N]=[a,b]$ is defined as follows:
\[
E_{\mathrm{tot},\mathrm{unif}}=\sum_{i=1}^N -\frac{1}{12}h^3f''(\xi_i).
\]
}

\begin{proposition}
Soient $a,b \in \mathbb{R}$, $N \in \mathbb{N}^*$, et $f:[a,b]\longrightarrow \mathbb{R}$ une fonction de classe  $C^2$, i.e., avec une dŽrivŽe seconde continue. Alors l'erreur totale $E_{\mathrm{tot},\mathrm{unif}}$ de la quadrature uniforme de $f$ par $N$ trapzes est bornŽe par: 
\begin{eqnarray}
E_{\mathrm{tot},\mathrm{unif}} \leq B_{\mathrm{tot},\mathrm{unif}}=\frac{1}{12}\frac{(b-a)^3}{N^2}\underset{x \in [a,b]}{\on{max}}|f''(x)|. \label{eq:upper1}
\end{eqnarray}
\end{proposition}

\delete{
\begin{proof}
Consider the interval $[a,b]$ and divide it into $N$ subintervals of length $h=\frac{b-a}{N}$. We set $x_0=a$ and $x_i=x_0+hi$ for $1 \leq i \leq N$. We approximate the function $f$ on each $[x_i, x_{i+1}]$ by a trapezoid, and so $p(x)$ is a piece-wise linear function given by the above equation on on each $1 \leq i \leq N$, there exists $\xi_i\in [x_{i-1}, x_{i}]$ such that the total integration error $E_{\mathrm{tot},\mathrm{unif}}$ on $[x_0, x_N]=[a,b]$ is
\[
E_{\mathrm{tot},\mathrm{unif}}=\sum_{i=1}^N -\frac{1}{12}h^3f''(\xi_i).
\]

By applying \cite[Theorem 20.5.1]{Hamming} to $-E_{\mathrm{tot},\mathrm{unif}}$ above, we see that there exists $\xi \in [\xi_1, \xi_N]$ such that
\[
E_{\mathrm{tot},\mathrm{unif}}=-\frac{1}{12}Nh^3f''(\xi)=-\frac{1}{12}(b-a)h^2f''(\xi).
\]
It follows that for this uniform choice of points $x_0, \dots, x_N$, the total quadrature error satisfies
\[
\displaystyle |E_{\mathrm{tot},\mathrm{unif}}| \leq B_{\mathrm{tot},\mathrm{unif}}
\]
where $B_{\mathrm{tot},\mathrm{unif}}=\frac{1}{12}\frac{(b-a)^3}{N^2}\underset{x \in [a,b]}{\on{max}}|f''(x)|$.
\end{proof}
}

\subsection{Quadrature basŽe sur la hessienne} \label{sec:refined}

 Plut™t que de choisir les points de quadrature sur une grille uniforme, nous proposons une {\it mŽthode raffinŽe} o la sŽlection s'adapte aux variations de la dŽrivŽe seconde de $f:[a,b] \to \mathbb{R}$ de classe $C^2$. 
 On divise $[a,b]$ en $k$ sous-intervalles $I_j$, $1 \leq j \leq k$, de longueur respective $l=\frac{b-a}{k}$. Pour permettre une comparaison avec la mŽthode uniforme, l'interpolation est faite telle que le nombre total de trapzes sur tous les $I_j$ est $N$.  
 On divise chaque $I_j$ en $n_j$  sous-intervalles o
\begin{equation}\label{eq:n_j}
   n_j=\biggl\lceil N \frac{\sqrt{M_j}}{ \sum_{p=1}^k \sqrt{M_p}}\biggl\rceil  
\end{equation}
avec $M_j=\underset{x \in I_j}{\on{max}}|f''(x)|$. 
En raison de la fonction $\lceil . \rceil$, $\sum_{j=1}^k n_j \approx \sum_{j=1}^k N \frac{\sqrt{M_j}}{ \sum_{p=1}^k \sqrt{M_p}}=N$ avec un Žcart au plus $k$,    celui-ci  devenant nŽgligeable quand $N \gg k$.
Pour chaque  $I_j$, on effectue une approximation de l'intŽgrale de $f_{\,\mkern 1mu \vrule height 2ex\mkern2mu I_j} : I_j \longrightarrow \mathbb{R}$ par $n_j$ trapzes puis on somme les rŽsultats. 
Nous pouvons dŽsormais Žnoncer  le rŽsultat principal:
\begin{theorem}\label{thm:tight}
Soient $a,b, \in \mathbb{R}$, $k, N \in \mathbb{N}^*$ avec $k \leq N$, et $f:[a,b]\longrightarrow \mathbb{R}$ une fonction de classe $C^2$. Alors le majorant $B_{\mathrm{tot},\mathrm{refined}} $ de l'erreur totale $E_{\mathrm{tot},\mathrm{refined}}$ de notre mŽthode de quadrature raffinŽe de l'intŽgrale de $f$ est plus fine que celle de Eq.~\eqref{eq:upper1} pour un mme nombre de trapzes:
\[
B_{\mathrm{tot},\mathrm{refined}} = \displaystyle\sum_{j=1}^k \frac{l^3}{12}\frac{1}{\bigg(\biggl\lceil N \frac{\sqrt{M_j}}{ \sum_{p=1}^k \sqrt{M_p}}\biggl\rceil\bigg)^2}|f''(\xi_j)| \leq B_{\mathrm{tot},\mathrm{unif}}.
\]
Pour chaque $1 \leq j \leq k$, $\xi_j$ est un ŽlŽment bien choisi de $I_j$. En particulier, plus $f''$ varie et plus l'inŽgalitŽ prŽcŽdente est stricte, et donc en faveur de notre mŽthode raffinŽe.
\end{theorem}

\begin{proof}[Preuve (sketch)]
Pour tout $1 \leq j \leq k$, on pose $h_j=\frac{l}{n_j}$ et on applique \cite[Th.20.5.1]{Hamming} pour montrer qu'il existe $\xi_j \in I_j$ tel que l'erreur totale de quadrature vŽrifie:
$
E_{\mathrm{tot},\mathrm{refined}}=\sum_{j=1}^k -\frac{1}{12}lh_j^2f''(\xi_j).
$
En posant 
$M=\underset{1 \leq j \leq k}{\on{max}}M_j$,  alors $E_{\mathrm{tot},  \mathrm{refined}}$ est  majorŽe par:
\begin{eqnarray}
B_{\mathrm{tot},\mathrm{refined}} & = & \displaystyle\sum_{j=1}^k \frac{l^3}{12}\frac{1}{\bigg(\biggl\lceil N \frac{\sqrt{M_j}}{ \sum_{p=1}^k \sqrt{M_p}}\biggl\rceil\bigg)^2}|f''(\xi_j)| \nonumber\\[5pt]
 & \leq &\displaystyle \frac{l^3}{12N^2} \sum_{j=1}^k \bigg(\sum_{p=1}^k\frac{\sqrt{M_p}}{\sqrt{M_j}}\bigg)^2|f''(\xi_j)|  \nonumber\\[5pt]
& \leq & \displaystyle \frac{l^3}{12N^2} \sum_{j=1}^k k^2 \frac{M}{M_j}|f''(\xi_j)| \nonumber\\[2pt]
 & \leq & \displaystyle \frac{l^3k^3}{12N^2} M = B_{\mathrm{tot},\mathrm{unif}}. \nonumber 
\end{eqnarray}

On note que la dernire inŽgalitŽ est obtenue en utilisant $|f''(\xi_j)| \leq M_j$, et donc plus $f''$ varie, plus cette inŽgalitŽ devient stricte en faveur de notre rgle de quadrature.
\end{proof}

\subsection{Illustration de la borne du ThŽorme \ref{thm:tight}}
Illustrons ici le comportement de notre mŽthode de quadrature pour la fonction $f(x)=\sin(\frac{1}{\sqrt{x}})$ sur $[0.1,1]$. Afin de dŽterminer $M_j$ pour chaque $1 \leq j \leq k$, on prend $S=100$ points rŽgulirement espacŽs $x_{j,s} \in I_j$, on calcule $M_j=\underset{1 \leq s \leq S}{\on{max}}|f''(x_{j,s})|$ et on dŽduit $n_j$ d'aprs Eq.~\eqref{eq:n_j}. Notons ici que le cožt de calcul de $M_j$ n'a pas d'importance. L'objectif est de montrer que sŽlectionner $N$ points ˆ partir de $f''$ donne une plus petite erreur de quadrature que d'utiliser $N$ points rŽgulirement espacŽs. En exploitant cette idŽe dans les PINNs, le mme budget de points de collocation sera bien sžr utilisŽ pour toutes les mŽthodes dans le processus d'entra"nement.

Pour la mŽthode uniforme, on sŽlectionne $N+1$ points rŽgulirement espacŽs entre $a$ et $b$ inclus. Ces points forment les extrŽmitŽs de $N$ trapzes. On calcule alors leurs aires respectives comme vu en Section \ref{sec:uniform} avant de les sommer.

\begin{figure}[h]
\begin{center}
\includegraphics[height=0.27\textwidth]{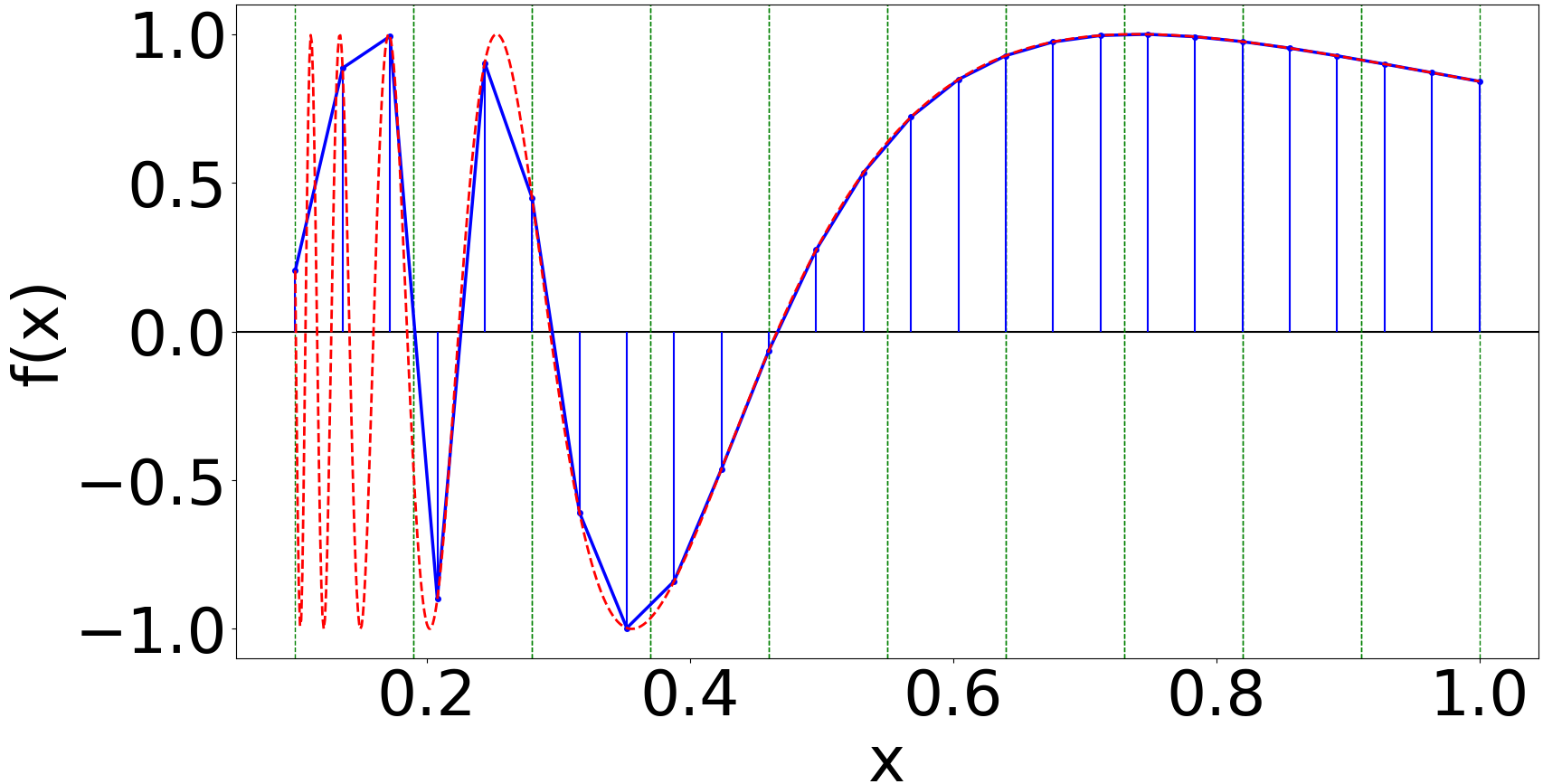}
\includegraphics[height=0.27\textwidth]{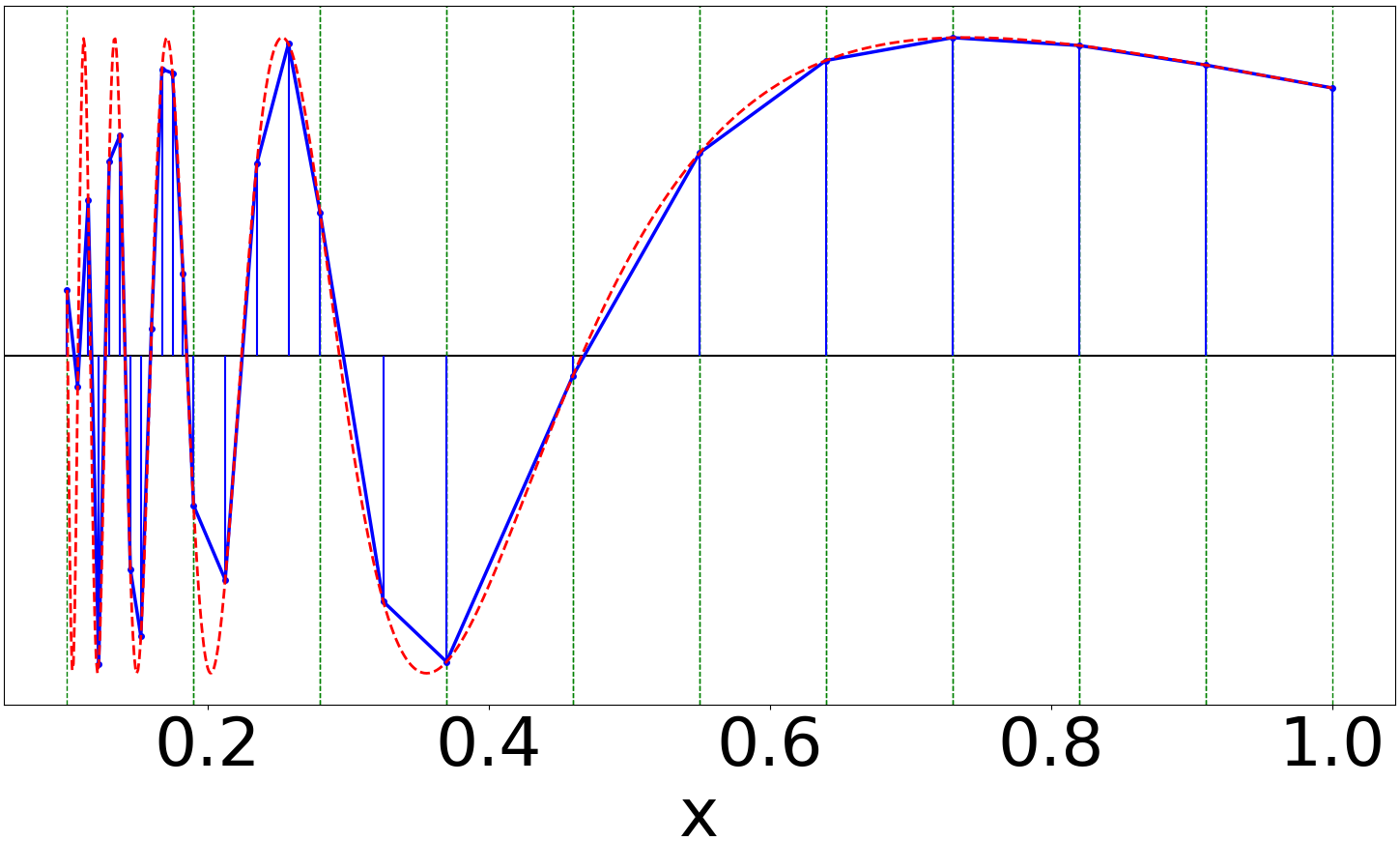}
\caption{$f(x)=\sin(\frac{1}{\sqrt{x}})$ (en rouge) et ses approximations (en bleu) avec $N=25$ ; (ˆ gauche) : mŽthode uniforme; (ˆ droite) : mŽthode raffinŽe avec $k=10$. \delete{Les erreurs relatives de quadrature sont $16.4 \%$ et $1.9 \%$ respectivement.}}
\label{fig:function_2}
\end{center}
\end{figure}

La Fig.~\ref{fig:function_2} illustre le comportement des deux mŽthodes de quadrature pour $N=25$ et $k=10$, et notamment celui pathologique de la mŽthode uniforme qui ne peut capturer les fortes variations de $f$ (courbe rouge) sur de petits intervalles. Au contraire, notre mŽthode utilise seulement une petite partie du budget pour approximer la partie droite de la fonction, et garde la majoritŽ des points de collocation pour les zones o les variations sont fortes, rŽduisant ainsi considŽrablement l'erreur de quadrature de $16.4 \%$ ˆ $1.89 \%$.

\section{ƒchantillonnage adaptatif et PINNs}\label{sec:PINN} 


Pour comparer les principales mŽthodes d'Žchantillonnage adaptatif pour les PINNs, nous utilisons le cadre de la mŽthode RAD~\cite{Chenxi2022}, o les $N$ points de collocation sont tirŽs selon une distribution proportionnelle ˆ un {\it critre d'intŽrt}. Ce dernier peut prendre la forme des {\bf rŽsiduels} de l'EDP comme dans la version originelle RAD \cite{Chenxi2022}, tre adaptŽ au {\bf gradient} des rŽsiduels  \cite{subramanian2023}, ˆ la {\bf hessienne} des rŽsiduels pour notre mŽthode, ou encore une {\bf  distribution uniforme} comme utilisŽ dans un PINN standard \cite{raissi2019physics}. 
Soit la distribution gŽnŽrique suivante:
\begin{eqnarray}
    d(\mathbf{x}) & \propto & \frac{\gamma(\mathbf{x})^{\tau}}{\mathbb{E}[\gamma(\mathbf{x})^{\tau}]}+c, \label{eq:PDF}
\end{eqnarray}
o $\tau$ et $c$ sont des hyperparamtres permettant de contr™ler la concentration des points. Les mŽthodes d'Žchantillonnage de l'Žtat de l'art peuvent toutes tres vues comme des cas spŽciaux de Eq.~\eqref{eq:PDF}, i.e., comme des instanciations d'un algorithme gŽnŽrique que nous appelons dans la suite $\bigstar$-RAD, o  res-RAD, grad-RAD, hessian-RAD, et unif-RAD, correspondent respectivement ˆ une mŽthode basŽe sur les rŽsiduels (i.e., o $\gamma(\mathbf{x})=f(\mathbf{x})$), leur gradient ($\gamma(\mathbf{x})=f'(\mathbf{x})$), la hessienne ($\gamma(\mathbf{x})=f''(\mathbf{x})$) et la distribution uniforme (PINN standard obtenu avec $\tau=0$ et $c \rightarrow \infty$). Si les dŽrivŽes de $f$ ne sont pas ˆ valeurs scalaires, on utilise la norme du vecteur ou de la matrice correspondante. Le pseudo-code de $\bigstar$-RAD est prŽsentŽ dans l'Algorithme~\ref{algo:*RAD}. Par ailleurs, dans ce qui suit, on utilise $\lambda_3=0$ (cf. Eq.~\eqref{eq:loss}).\\
Nous prŽsentons dans ce qui suit les rŽsultats obtenus avec les quatre mŽthodes d'Žchantillonnage sur deux EDPs en 2D\delete{: l'Žquation de Poisson et l'Žquation de diffusion-rŽaction}\footnote{Le code est disponible sur ce \href{https://github.com/Antoine-ml-code/Adaptive-Sampling-for-Collocation-Points-in-PINNs-ECML-2025.git}{dŽp™t GitHub}.}.

\begin{algorithm}[t]
\caption{$\bigstar$-RAD} \label{algo:*RAD}
\begin{algorithmic}[1]
\State Fixer $\bigstar$ $\in \{ ``\text{res}", ``\text{grad}", ``\text{hessian}", ``\text{unif}"\}$, $\tau$, $c$, $N$ et $\#epochs$;
\State SŽlectionner alŽatoirement un ensemble de points $S$;
\State Entra"ner un PINN pour un nombre d'Žpoques donnŽ;
\While{$\#epochs$ n'est pas atteint}
    \State Construire une distribution $d(\mathbf{x})$ de Eq.~\eqref{eq:PDF} pour $\bigstar$ ˆ partir d'un ensemble de points alŽatoires;
    \State $S \leftarrow$ nouvel ensemble de $N$ points i.i.d. $\sim d(\mathbf{x})$;
    \State Entra"ner le PINN pour un nombre d'Žpoques;
\EndWhile
\end{algorithmic}
\end{algorithm}

\subsection{ƒquation de Poisson 2D} \label{sec:poisson}
L'Žquation de Poisson est une EDP elliptique du second ordre utilisŽe en physique thŽorique et dŽfinie comme suit: $
\Delta u=F(x,y),$ o $(x,y) \in [0,1]^2$, et o $F$ est prise telle que  $u(x,y)=2^{4a}x^a(1-x)^ay^a(1-y)^a$ avec $a=10$ est la solution analytique (Fig.~\ref{fig:Poisson} en haut ˆ gauche).
Un PINN est appris en utilisant la fonction d'activation {\tt tanh} sur un rŽseau entirement connectŽ et composŽ de 3 couches cachŽes de 20 neurones. Les paramtres suivants sont utilisŽs:  $\#epochs=20000$, le taux d'apprentissage $\eta=10^{-3}$, $N=400$ tirŽs selon $d(\mathbf{x})$ approximŽe ˆ partir de $40000$ candidats, $\tau=1/2$  et $c=0$. Le rŽŽchantillonnage est rŽalisŽ toutes les 1000 itŽrations.

La remarque la plus frappante que l'on peut faire ˆ partir de la Fig.~\ref{fig:Poisson} (en haut ˆ droite) est que notre mŽthode hessian-RAD profite pleinement des variations abruptes des solutions de Poisson pour converger beaucoup plus rapidement que les autres. Environ 1000 itŽrations suffisent, tandis que les stratŽgies concurrentes nŽcessitent beaucoup plus d'itŽrations pour se stabiliser. Fait intŽressant, mme aprs 20000 itŽrations, lorsque les mŽthodes ont convergŽ vers une solution exacte, l'Žcart en termes d'erreur de prŽdiction en faveur de notre mŽthode est important, comme l'illustrent les quatre heatmaps de la Fig.~\ref{fig:Poisson} (partie infŽrieure). Les erreurs $L_2$ sont $6 \times 10^{-5}$, $5 \times 10^{-6}$, $2 \times 10^{-6}$ et $7 \times 10^{-7}$ respectivement pour unif-RAD, res-RAD, grad-RAD et hessian-RAD. 
Bien que le calcul de la hessienne soit plus cožteux, cette charge supplŽmentaire est raisonnable et donc compensŽe par une meilleure prŽdiction.

\begin{figure}[t]
\includegraphics[width=0.8\textwidth]{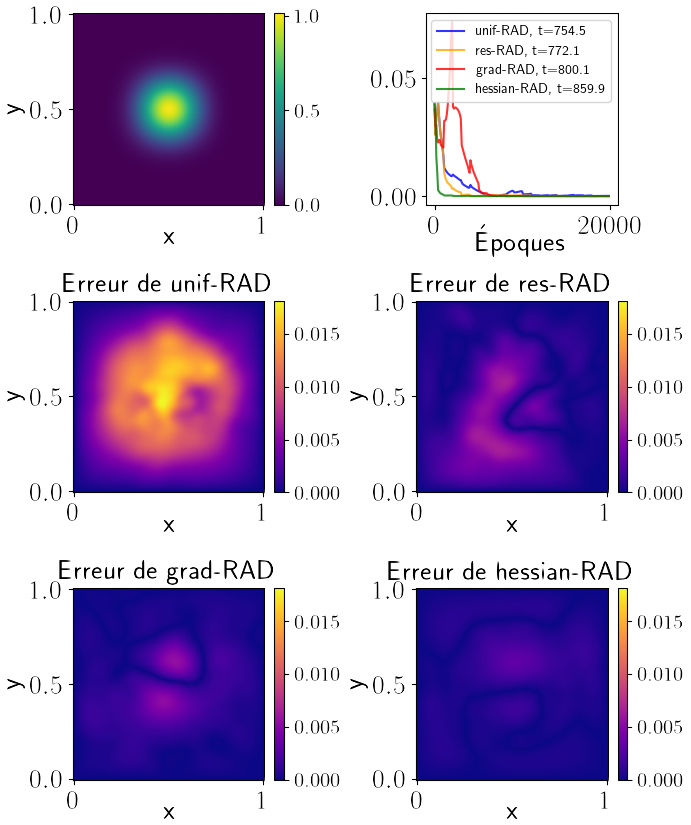}
\caption{(Haut gauche) Solution analytique de l'Žquation de Poisson ; (Haut droit) Erreur test $L_2$ au cours des 20000 premires itŽrations, et temps de calcul (en s) ; (Milieu et bas) Heatmaps des erreurs des 4 mŽthodes aprs 20000 itŽrations.}
\label{fig:Poisson}
\end{figure}

\subsection{ƒquation de diffusion-rŽaction 2D} 
L'Žquation de diffusion-rŽaction est dŽfinie par : $\frac{\partial u}{\partial t}=D\frac{\partial^2 u}{\partial x^2}+F(x,t)$, $x \in [-\pi, \pi]$, $t \in [0,1]$, $u(x,t)$ est la  concentration de solutŽ, $D=1$ reprŽsente le coefficient de  diffusion, et $F(x,t)=e^{-t}\big[\frac{3}{2}\sin(2x)+\frac{8}{3}\sin(3x)+\frac{15}{4}\sin(4x)+\frac{63}{8}\sin(8x) \big]$ est la rŽaction chimique. Les conditions initiales et aux bords sont $u(x,0)=\sum_{i=1}^4\frac{\sin(ix)}{i}+\frac{\sin(8x)}{8}$ et $u(-\pi,t)=u(\pi,t)=0$. La solution  est dŽcrite sur la Fig.~\ref{fig:Diffusion2} (haut gauche).
La mme architecture de PINN qu'en \ref{sec:poisson} est utilisŽe avec les paramtres $\#epochs=100000$, $\eta=10^{-4}$, $N=50 \sim d(\mathbf{x})$ approximŽe ˆ partir de $5000$ candidats, $\tau=1/2$  et $c=0$.\\ 
ƒtonnamment sur cette EDP, le niveau des rŽsiduels n'est que trs peu informatif pour bien Žchantillonner les points de collocation, probablement du fait d'un {\it loss landscape} trs peu lisse (contrairement ˆ la Section \ref{sec:poisson}) liŽ aux fortes variations de la solution. La mŽthode res-RAD est ainsi bien moins efficace que les autres avec une erreur 5 ˆ 10 fois supŽrieure ($3 \times 10^{-2}$), alors que celles d'unif-RAD, grad-RAD et hessian-RAD sont $6 \times 10^{-3}$, $3 \times 10^{-3}$ et $3 \times 10^{-3}$ respectivement. Les deux mŽthodes basŽes sur les dŽrivŽes d'ordre 1 et 2 (voir heatmaps d'erreur sur Fig.~\ref{fig:Diffusion2} - bas)  capturent bien mieux les variations de la loss, tout comme la mŽthode uniforme (haut droite) qui est un peu moins performante mais permet nŽanmoins de bien couvrir le domaine. On notera que les 3 distributions d'erreurs sont trs diffŽrentes, celle d'unif-RAD tendant ˆ tre rŽpartie sur l'ensemble du domaine, celles de grad-RAD et hessian-RAD Žtant plus localisŽes autour de zones spŽcifiques.

\delete{
\begin{figure}[t]
\includegraphics[width=0.47\textwidth]{Differences Diffusion-reaction.png}
\caption{Heatmaps of errors of the 4 sampling methods after 20000 epochs for the Diffusion-reaction PDE.}
\label{fig:Diffusion1}
\end{figure}
}
\begin{figure}[t]
\includegraphics[width=0.8\textwidth]{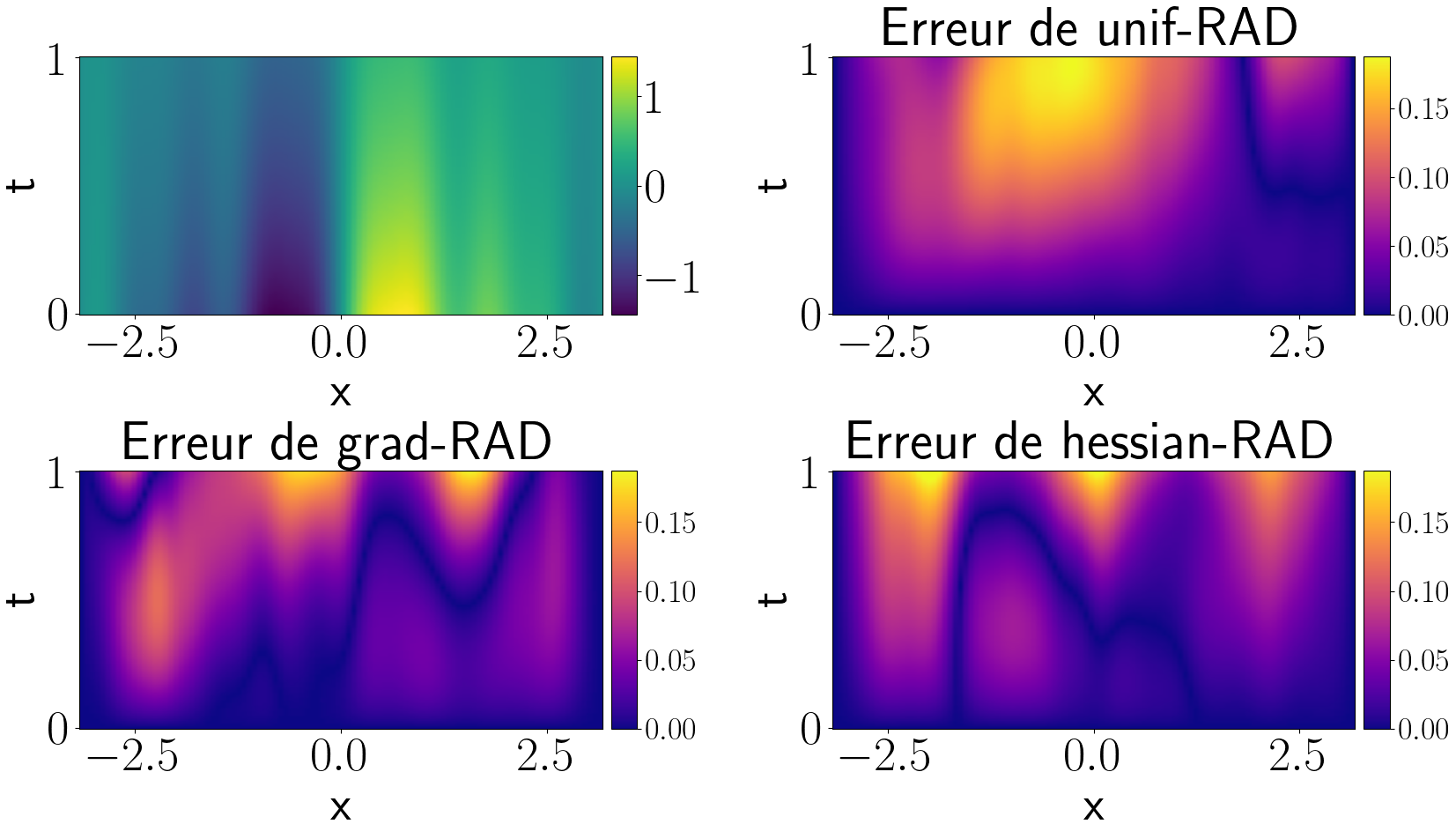}
\caption{(Haut gauche) Solution analytique de l'Žquation de diffusion-rŽaction ; (Autres) Heatmaps des erreurs des mŽthodes basŽes sur les rŽsiduels aprs 100000 itŽrations.}
\label{fig:Diffusion2}
\end{figure}


\section{Conclusion}

Nous avons prŽsentŽ une mŽthode de quadrature basŽe sur les dŽrivŽes secondes. L'exploitation de la hessienne des rŽsiduels montre Žgalement des rŽsultats prometteurs dans une mŽthode d'Žchantillonnage adaptatif pour les PINNs. 
Mais l'utilisation de $f''$ peut devenir cožteuse en haute dimension. Une  direction porte sur une approche stochastique, o ˆ chaque  rŽŽchantillonnage, les ŽlŽments de la hessienne ˆ calculer seraient ŽchantillonnŽs. Une autre direction consiste ˆ s'appuyer sur le fait que des mŽthodes comme gPINN~\cite{gPINN2021} rŽalisent dŽjˆ une grande partie des calculs nŽcessaires pour la hessienne, et peuvent donc tre combinŽes avec peu de cožt supplŽmentaire. 

\medbreak

\noindent \textbf{Remerciements.}
Ce travail a ŽtŽ financŽ par l'ANR sous le projet ``France 2030'', avec la rŽfŽrence EUR MANUTECH SLEIGHT - ANR-17-EURE-0026.

\renewcommand{\refname}{RŽfŽrences}

\delete{
\bibliography{ecml-biblio}
}

\end{document}